\documentclass[letterpaper, 10 pt, conference, twoside]{ieeeconf}

\IEEEoverridecommandlockouts
\usepackage[utf8]{inputenc}
\usepackage{blindtext}
\usepackage{tabularx}

\usepackage{amsthm}

\usepackage{mathtools}
\usepackage{amssymb}
\usepackage{cite}
\usepackage{lipsum}
\usepackage{xcolor}
\usepackage{tikz}
\usepackage{pgfplots}
\usepackage{svg}
\usepackage{adjustbox}
\usepackage{comment}
\usepackage{standalone}
\usepackage{calc}
\usepackage[acronym]{glossaries}
\usepackage{algorithm}
\usepackage{algpseudocode}
\usepackage{ifthen}
\newboolean{proofs}
\usepackage{aliascnt}

\usepackage{siunitx}

\usepackage[
 colorlinks,
    linkcolor={blue!50!black},
    citecolor={blue!50!black},
    urlcolor={blue!80!black}]{hyperref}
    
\usepackage{multicol}
\usepackage{todonotes}
\presetkeys{todonotes}{inline}{}

\usepackage[capitalize]{cleveref}
\Crefname{equation}{Eq.}{Eqs.}
\Crefname{figure}{Fig.}{Figs.}
\Crefname{tabular}{Tab.}{Tabs.}
\Crefname{definition}{Def.}{Defs.}
\Crefname{section}{Sec.}{Sects.}
\Crefname{theorem}{Thm.}{Thms.}
\Crefname{condition}{Cond.}{Conds.}

\usepackage{graphicx}
\usepackage{booktabs}
\usepackage{threeparttable}
\usepackage{multirow}
\usepackage{tikz, pgfplots, pgfplotstable}
\usetikzlibrary{
    positioning,
  	cd,
  	shapes,
  	math,
  	decorations.markings,
  	decorations.pathmorphing,
	decorations.pathreplacing,
  	arrows.meta,
  	calc,
  	fit,
  	quotes,
  }
\tikzcdset{arrow style=tikz, diagrams={>=To}}
\pgfplotsset{compat=1.15}
\let\labelindent\relax

\tikzstyle{block} = [draw, rectangle, minimum height=2em, minimum width=3em,thick]
\tikzstyle{blockdot} = [block, dotted,rounded corners=4, inner sep=-2pt]
\tikzstyle{blockfill} = [block,rounded corners=4, inner sep=-2pt,fill=blue!5!white]
\tikzstyle{every node}=[font=\footnotesize]

\newcommand{\NoNumberInput}[1]{%
  \item[] \textbf{Input:} #1
}

\makeatletter
\let\MYcaption\@makecaption
\makeatother
\usepackage[font=footnotesize]{subcaption}
\makeatletter
\let\@makecaption\MYcaption
\makeatother

\newtheorem{theorem}{Theorem}
\newcommand{\newaliasthm}[2]{%
  \newaliascnt{#1}{theorem}%
  \newtheorem{#1}[#1]{#2}%
  \aliascntresetthe{#1}}
\newaliasthm{corollary}{Corollary}
\newaliasthm{axiom}{Axiom}
\newaliasthm{claim}{Claim}
\newaliasthm{assumption}{Assumption}
\newaliasthm{condition}{Condition}
\newaliasthm{example}{Example}
\newaliasthm{remark}{Remark}
\newaliasthm{problem}{Problem}

\newaliascnt{definition}{theorem}
\newtheorem{definition}[definition]{Definition}
\aliascntresetthe{definition}

\newaliascnt{proposition}{theorem}
\newtheorem{proposition}[proposition]{Proposition}
\aliascntresetthe{proposition}

\newaliascnt{lemma}{theorem}
\newtheorem{lemma}[lemma]{Lemma}

\DeclareMathOperator{\dist}{dist}
\crefname{assumption}{Assumption}{Assumptions}
\Crefname{assumption}{Assumption}{Assumptions}
\aliascntresetthe{lemma}

\crefname{enumi}{condition}{conditions}
\Crefname{enumi}{Condition}{Conditions}
\crefname{theorem}{theorem}{theorems}
\Crefname{theorem}{Theorem}{Theorems}
\Crefname{definition}{Definition}{Definitions}
\Crefname{proposition}{Proposition}{Propositions}
\crefname{remark}{remark}{remarks}
\Crefname{remark}{Remark}{Remarks}

\newcommand{\la}{\lambda}

\newcommand{\bbN}{\mathbb{N}}

\newcommand{\bbR}{\mathbb{R}}

\newcommand{\cH}{\mathcal{H}}

\newcommand{\cR}{\mathcal{R}}

\newcommand{\cX}{\mathcal{X}}

\definecolor{r1}{RGB}{12,7, 134}
\definecolor{r2}{RGB}{126, 3, 167}
\definecolor{r3}{RGB}{203, 71, 119}
\definecolor{r4}{RGB}{248, 149, 64}

\newcommand{\rulebook}{\langle \cR, \leq \rangle}
\newcommand{\Fg}{\ensuremath{F_g}}

\newacronym{av}{AV}{Autonomous Vehicle}
\newacronym{ibr}{IBR}{iterated-best-response}
\newacronym{nash}{NE}{Nash Equilibria}
\newacronym{cudg}{CUDG}{Communal Urban Driving Game}
\newacronym{cudgs}{CUDGs}{Communal Urban Driving Games}
\newacronym{gnep}{GNEP}{Generalized Nash Equilibrium Problem}
\newacronym{abk:av}{AV}{Autonomous Vehicle}
\newacronym{ad}{AD}{Autonomous Driving}
\newacronym{abk:br}{BR}{Best Response}
\newacronym{vru}{VRU}{Vulnerable Road Users}
\newacronym{lmop}{LMOP}{Lexicographic Multi-Objective Problem}

\DeclareMathOperator*{\argmin}{arg\,min}

\definecolor{baiocchi}{RGB}{193,221,245}

\newcommand{\spref}{\prec}                           %
\newcommand{\wpref}{\precsim}                        %
\newcommand{\indiff}{\sim}                           %
\newcommand{\wprefrb}{\wpref_{\rulebook}}            %
\newcommand{\lexleq}{\preceq_{\mathrm{lex}}}         %
\newcommand{\ordX}{\langle \cX, \wpref \rangle}      %
\newcommand{\viol}{R}                                %
\newcommand{\Ft}{\tilde{F}}                          %
\newcommand{\Lt}{\tilde{L}}                          %
\newcommand{\Gbar}{\bar{G}}                          %
\newcommand{\dsum}[1]{\delta_{#1}}                   %
\newcommand{\dsumj}[1]{\delta^{(j)}_{#1}}            %
\newcommand{\wfold}{w^\star}
\newcommand{\Kbar}{\bar{K}}                          %
\newcommand{\Sof}[1]{S^{#1}}                         %
\newcommand{\Sidx}[1]{S_{#1}}                        %
\newcommand{\Sk}{\Sidx{k}}
\newcommand{\Szero}{\Sof{0}}                          %
\newcommand{\Slim}{\Sidx{\lim}}                      %
\newcommand{\Gof}[1]{G^{#1}}                         %
\newcommand{\Gidx}[1]{G_{#1}}                        %
\newcommand{\Gk}{\Gidx{k}}
\newcommand{\Gzero}{\Gof{0}}                         %
\newcommand{\proj}{\Pi_{\cX}}                        %
\newcommand{\projon}[1]{\Pi_{#1}}                    %
\newcommand{\tcone}[1]{\mathcal{T}_{\cX}(#1)}                  %
\newcommand{\xinit}{x_{\mathrm{init}}}
\newcommand{\xstar}{x^\star}
\newcommand{\InnerLoop}{\texttt{InnerLoop}}
\newcommand{\SolverStep}{\texttt{SolverStep}}
\newcommand{\Converged}{\texttt{Converged}}

\newcommand{\Rpos}{\bbR_{>0}}
\newcommand{\Rnn}{\bbR_{\geq 0}}
\newcommand{\Npos}{\bbN_{\geq 1}}
\newcommand{\limk}{\lim_{k \to \infty}}
\newcommand{\liminfk}{\liminf_{k \to \infty}}
\newcommand{\defeq}{\coloneqq}

\usepackage{ifxetex}
\ifxetex
\usepackage[tracking=false,kerning=false,spacing=false]{microtype}
\usepackage[protrusion=true,tracking=false,kerning=false,spacing=false]{microtype}
\else
\usepackage[tracking=false,kerning=true,spacing=true]{microtype}

\fi
\usepackage{enumitem}
\renewcommand{\baselinestretch}{1.0}
\begin{document}
\setboolean{proofs}{true}
\bstctlcite{IEEEexample:BSTcontrol}
\title{\LARGE \bf
    Lexicographic Descent: A Gradient-Based Approach to Optimize Totally Ordered Rulebooks via Parametric Scalarization
}

\author{Matteo Penlington, Alessandro Zanardi, Lorenzo Fiaschi, Emilio Frazzoli%
\thanks{This work was supported as a part of NCCR Automation, a National Centre of Competence in Research, funded by the Swiss National Science Foundation (grant number 51NF40\_225155).}%
\thanks{
M. Penlington, A. Zanardi, L. Fiaschi and E. Frazzoli are with the Institute for Dynamic Systems and Control, ETH Z\"urich, Switzerland {\tt mpenlington@ethz.ch}.}%
}

\maketitle
\begin{abstract}
Safety-critical systems, such as autonomous vehicles, must satisfy conflicting constraints and objectives, collectively called rules, typically ordered lexicographically. 
Preemptive lexicographic optimization solves a sequence of problems, one per priority level. However, higher levels are optimized without considering lower ones, making the approach susceptible to suboptimal solutions. 
On the other hand, scalarization optimizes all levels jointly and so avoids this myopia, but a finite weighted sum permits trade-offs across priority levels, and exact scalar representations lack the regularity standard solvers require. 
We propose Lexicographic Descent, which combines the benefits of both: a single weighted sum whose relative weights decay at priority-dependent rates, recovering the lexicographic order in the limit. We show that Lexicographic Descent converges to constrained stationary points of the highest-priority rule, matching preemptive guarantees. 
Since the lower-priority weights vanish, we derive a necessary condition for the lower-priority rules to keep guiding the iterates, with the aim of avoiding suboptimal stationary points of the higher-priority rules.
We show that Lexicographic Descent can exploit problem structure to reach the lexicographic optima, and outperforms the preemptive approach on an autonomous driving case study.
\end{abstract}
\IEEEpeerreviewmaketitle
\section{Introduction} \label{sec:introduction}

Safety-critical systems, such as~\glspl{av}, must contend with conflicting constraints and objectives. An~\gls{av} must reach its destination safely, while adhering to road regulations, cultural norms, and contending with ethical dilemmas~\cite{DeFreitas2021}. 
For example, avoiding a collision is incomparably more important than violating a traffic rule.
Such priorities are often expressed as a lexicographic order over constraints and objectives, which we collectively call \emph{rules}.
To synthesize behaviors that adhere to this order, lexicographic optimization~\cite{Zarepisheh2011} has been used in~\glspl{av} and other safety-critical control systems~\cite{Pavlasek2024,Halder2023,Wang2020,Zanardi2022}.

Preemptive approaches~\cite{Sherali1983} sequentially solve a constrained problem for each rule, adding constraints so that the returned solutions do not increase the violation of more important rules. Since each level is optimized without regard to the levels below it, the approach is prone to suboptimal solutions.

Non-preemptive approaches instead scalarize the lexicographic problem~\cite{Cococcioni2018}, optimizing all rules jointly and thereby avoiding this myopia. The most common is a finite weighted sum~\cite{Wilde2024}, which, outside special settings, such as discrete sets with linear objectives~\cite{Sherali1983}, permits trade-offs across priority levels. Conversely, exact scalar representations of a lexicographic order need not exist---the lexicographic order over $\bbR^2$ admits none~\cite{Debreu1954}---and where they do, they are in general discontinuous~\cite{Fishburn1970} so standard gradient-based solvers do not apply. Scalarizations with infinitesimal weights via non-standard analysis avoid both issues~\cite{Cococcioni2024,Fiaschi2022}, but their computational burden and numerical stability have limited their application thus far. Techniques for handling constraints or local optima, such as penalty~\cite{Nocedal2006} and continuation methods~\cite{Allgower2003}, can be incorporated as components of both preemptive and non-preemptive approaches, and are therefore complementary to this work.

We propose a non-preemptive approach that combines the lexicographic order of preemptive approaches with the joint optimization of all rules from non-preemptive ones. 

Specifically, our contributions are:
\begin{enumerate}[label=(\roman*), leftmargin=*, nosep]
    \item We extend the representation of ordered sets by scalar functions to \emph{parametric} scalar functions, in which the order is recovered in the limit of a scalar parameter.
    \item We construct a weighted sum whose weights decay at priority-dependent rates and prove that it parametrically represents the lexicographic order while retaining the regularity required by gradient-based optimizers.
    \item We introduce Lexicographic Descent, an algorithm that interleaves solver steps with weight updates, and prove that its iterates converge to the constrained stationary set of the highest-priority rule.%
    \item We derive a necessary condition on the weight updates for the lower-priority rules to influence the solution.
\end{enumerate}
We validate the effectiveness of Lexicographic Descent on two case studies. 
The first is a toy problem, on which Lexicographic Descent leverages the underlying problem structure to converge to the lexicographic optimum; characterizing the problem structures and sufficient conditions for which this holds is left for future work. The second is an autonomous driving setup, where we show how Lexicographic Descent outperforms the preemptive approach.

\section{Representation of Ordered Sets} \label{sec:representation_of_ordered_sets}
The ensuing formalization builds upon the notion of representing ordered sets via real-valued scalar functions~\cite{Beardon2002}, the language specification introduced in rulebooks~\cite{Censi2019}, and on the related work on minimum-violation planning~\cite{Wongpiromsarn2021}. We first introduce the general parametric representation of ordered sets, and then present a candidate function that captures the preferences induced by prioritized constraints.

Let $\cH$ be a finite-dimensional Hilbert space, and let $\cX \subset \cH$ be a subset.
For an~\gls{av}, $\cX$ could represent all available actions or trajectories over a specified time horizon.

\subsection{Parametric Representation of Ordered Sets}\label{sec:param_rep}
This work relies on using scalar real-valued functions to faithfully represent ordered sets.
Specifically, given $x,y\in\cX$, their relative preference under a binary relation $\wpref$ should be recoverable by comparing the values of a scalar function. The notion of representing an ordered set by a real-valued scalar function is formalized below.
\begin{definition}[Representability~\cite{Beardon2002}] \label{def:representability}
    An ordered set $\ordX$ is representable if there exists 
    $f:\cX \rightarrow \bbR$ such that
    \begin{equation}
        \begin{aligned}
            x \spref y &\iff f(x) < f(y), \quad \forall x,y \in \cX,\\
            x \indiff y &\iff f(x)=f(y), \quad \forall x,y \in \cX.
        \end{aligned}
    \end{equation}
\end{definition}

Modern optimization solvers typically require regularity conditions on the objective (such as continuity or differentiability) to obtain convergence guarantees. This presents a practical challenge. Not all orderings can be represented by a scalar function; in particular, the lexicographic order on $\bbR^d$, $d \ge 2$, admits no real-valued representation at all, let alone a continuous one~\cite{Debreu1954}.

Although continuous representations can be obtained under additional structural assumptions~\cite{Fishburn1970,Shi2020}, these require restrictive modifications to the underlying decision space, such as restricting the decision space to a countable set. To recover the ordering without imposing such assumptions, this work introduces a scalar function that \emph{parametrically} represents an ordered set $\ordX$.

\begin{definition}[Parametric Representation of $\ordX$]
    \label{def:parametric_representability}
    A function $F : \cX \times \bbN \to \bbR$
    is a parametric representation of $\ordX$ if
    \begin{equation}
        \begin{aligned} \label{eq:param_repr_condt}
            x \spref y
            \iff {} & \exists k_0 \in \bbN \text{ such that } \\
                    & F(x,k) < F(y,k), \quad \forall k > k_0, \\
            x \indiff y
            \iff {} & \exists k_0 \in \bbN \text{ such that } \\
                    & F(x,k) = F(y,k), \quad \forall k > k_0.
        \end{aligned}
    \end{equation}
\end{definition}

\subsection{Candidate Parametric Function}
We now construct a candidate parametric function for ordered sets induced by prioritized objectives and constraints. To capture this structure, we formalize individual objectives and constraints as \emph{rules}.
\begin{definition}[Rule~\cite{Censi2019}] \label{def:rule}
    A rule is a function ${r: \cX\rightarrow\Rnn}$. 
    It maps a decision $x\in\cX$ to a non-negative real number that represents the degree to which $x$ violates the rule. A decision $x$ is said to \emph{satisfy} rule $r$ if $r(x)=0$.
\end{definition}
For~\glspl{av}, a rule could be any form of regulatory function (such as a speed limit), cultural behavior (such as a Pittsburgh left), or a specified objective (such as comfort or lane centering). More examples can be found in~\cite{Collin2020,Karnchanachari2024}.
Autonomous systems are often subject to multiple rules simultaneously, with varying importance. We formalize this hierarchical collection as a \emph{rulebook}~\cite{Censi2019}.
\begin{definition}[Rulebook] \label{def:rulebook}
    A rulebook $\rulebook$ is a finite set $\cR$ of rules equipped with a total order $\le$, where $r \le r'$ means that $r$ is at least as important as $r'$.
\end{definition}
The $N$ rules are indexed $0,\dots,N-1$ in decreasing importance, so lower indices identify more important rules.
A rulebook induces a total pre-order over the decision space. This is the ordering that should be represented by a parametric function.
\begin{definition}[Rulebook-Induced Order on $\cX$] \label{def:rulebook-induced-order}
    Define the violation vector of a decision $x \in \cX$ as
    \[
        \viol(x) \defeq (r_0(x),\dots,r_{N-1}(x)) \in \Rnn^N.
    \]
    The rulebook induces a total pre-order $\wprefrb$ on $\cX$ by
    \[
        x \wprefrb y 
        \;\;\iff\;\;
        \viol(x) \lexleq \viol(y),
    \]
    where $\lexleq$ denotes the lexicographic order on $\bbR^N$.
\end{definition}
\noindent When clear from context, we write $\wpref$ instead of $\wprefrb$.

We now introduce a candidate parametric function based on weighted sums of rule violations. The key idea is that the parameter $k$ governs the relative importance of the rules. By requiring that, as $k$ increases, the weights associated with lower-priority rules become negligible relative to higher-priority rules, we ensure that, for any two decisions, lower-priority violations cannot induce trade-offs with higher-priority rules for all sufficiently large $k$.
\begin{theorem}[Parametric Function of $\ordX$] \label{thm:function_class}
    Given the order $\ordX$ from~\Cref{def:rulebook-induced-order} and its associated violation vector $\viol(x)$, let $g : \bbN \to \Rpos^N$ be a function such that, for all $i = 0,\dots,N-2$,
    \begin{equation} \label{eq:limit_cond}
        \limk \frac{g_{i+1}(k)}{g_i(k)} = 0.
    \end{equation}
    Then the function $\Fg : \cX \times \bbN \to \Rnn$ defined by,
    \[
        \Fg(x,k) = g(k)^\top \viol(x),
    \]
    parametrically represents $\ordX$.
\end{theorem}
\begin{proof}
    To show that $\Fg$ parametrically represents $\ordX$, we must prove the forward and reverse implications of~\Cref{eq:param_repr_condt}. We divide the proof into four parts. Note first that, for every $\ell > j$, $\lim_{k \to \infty}\frac{g_\ell(k)}{g_j(k)} = \lim_{k \to \infty} \prod_{i=j}^{\ell-1} \frac{g_{i+1}(k)}{g_i(k)} = 0$ by~\eqref{eq:limit_cond}.

    \emph{Case 1:} $x \indiff y \implies \exists k_0 \text{ such that } \forall k > k_0, \; \Fg(x,k) = \Fg(y,k)$. Assume $x \indiff y$. By~\Cref{def:rulebook-induced-order}, this implies that $\viol(x)=\viol(y)$. It follows that
    $\Fg(x,k) = \Fg(y,k), \, \forall k,$
    which satisfies the condition for any $k_0$.

    \medskip
    \emph{Case 2:} $x \indiff y \impliedby \exists k_0 \text{ such that } \forall k > k_0, \; \Fg(x,k) = \Fg(y,k)$. Suppose $\Fg(x,k)=\Fg(y,k)$ for all $k>k_0$. 
    Then $g(k)^\top (\viol(x)-\viol(y))=0$ for all $k>k_0$. 
    Let $\Delta := \viol(x)-\viol(y)$, with components 
    $\Delta_i = r_i(x)-r_i(y)$.

    Assume for contradiction that $\Delta \neq 0$, and let $j \defeq \min\{i \in \{0,\dots,N-1\} \mid \Delta_i \neq 0\}$. Expanding the dot product yields
    \[
        g_j(k)\Delta_j + \sum_{\ell=j+1}^{N-1} g_\ell(k)\Delta_\ell = 0, \quad \forall k > k_0.
    \]
    Dividing by $g_j(k)$, we obtain
    \[
        \Delta_j + \sum_{\ell=j+1}^{N-1} \frac{g_\ell(k)}{g_j(k)}\Delta_\ell = 0, \quad \forall k > k_0.
    \]
    Because this equality holds for all $k > k_0$, it must also hold in the limit as $k \to \infty$. Taking the limit of both sides and using~\Cref{eq:limit_cond} yields
    \begin{align}
        0 &= \limk \left(\Delta_j + \sum_{\ell=j+1}^{N-1} \frac{g_\ell(k)}{g_j(k)}\Delta_\ell \right) = \Delta_j.
    \end{align}
    This contradicts the assumption. Therefore, $\Delta = 0$, meaning $\viol(x) = \viol(y)$, and thus $x \indiff y$.

    \medskip
    
    \emph{Case 3}: $x \spref y \implies \exists k_0 \text{ such that } \forall k > k_0, \; \Fg(x,k) < \Fg(y,k)$. Assume $x \spref y$. Define $\Delta_i$ and $j$ as in Case 2. 
    \begin{align*}
        \Fg(x,k)-\Fg(y,k)
        &= g_j(k)\left(\Delta_j + \sum_{\ell=j+1}^{N-1}
        \frac{g_\ell(k)}{g_j(k)}\,\Delta_\ell\right).
    \end{align*}
    Since the sum is finite and each $\Delta_i$ is constant, it suffices to show that the bracketed term is negative for all sufficiently large $k$. Taking the limit as $k \to \infty$ and using~\Cref{eq:limit_cond} yields
    \begin{align}
        \limk & \left( \Delta_j + \sum_{\ell = j+1}^{N-1}\frac{g_\ell(k)}{g_j(k)}\Delta_\ell\right) = \Delta_j
    \end{align}
    Since $x \spref y \iff r_j(x) < r_j(y)$, it follows that $\Delta_j < 0$. Therefore, by definition of a limit, there exists $k_0$ such that for all $k > k_0$, $\Fg(x,k)<\Fg(y,k)$.
    
    \medskip
    
    \emph{Case 4:} $x \spref y \impliedby \exists k_0 \text{ such that } \forall k > k_0, \; \Fg(x,k) < \Fg(y,k)$.
    Assume there exists $k_0$ such that $\Fg(x,k) < \Fg(y,k)$ for all $k>k_0$. This implies $\viol(x) \neq \viol(y)$, so $\Delta \neq 0$.
    Define $\Delta_i$ and $j$ similarly.
    Since $\Fg(x,k)-\Fg(y,k)<0$ for all $k>k_0$ and $g_j(k)>0$ always, it follows that
    \[
        \Delta_j + \sum_{\ell=j+1}^{N-1}\frac{g_\ell(k)}{g_j(k)}\,\Delta_\ell < 0,
        \qquad \forall k>k_0.
    \]
    By~\Cref{eq:limit_cond},
    \[
        \limk \left(\Delta_j + \sum_{\ell=j+1}^{N-1}\frac{g_\ell(k)}{g_j(k)}\,\Delta_\ell\right) \leq 0 \implies \Delta_j \leq 0.
    \]
    Since we established that $\Delta_j \neq 0$, it must be that $\Delta_j < 0$. Therefore, 
    $
    r_j(x)-r_j(y)<0 \quad\Longleftrightarrow\quad r_j(x)<r_j(y),
    $
    and thus $x \spref y$.
\end{proof}

\section{Optimization of Parametric Representations} \label{sec:opt_rulebooks}
The goal of this work is to find the minimal elements of the rulebook-induced order $\wprefrb$ of~\Cref{def:rulebook-induced-order}, that is, the decisions to which no other decision is strictly preferred. Since $\wprefrb$ is a relation between pairs of decisions, it is not itself amenable to optimization directly. However, since $\Fg$ parametrically represents $\wprefrb$ (\Cref{thm:function_class}), comparing any two decisions $\Fg(x,k) < \Fg(y,k)$ recovers $x \spref y$ for all $k > k_0$. Thus, a decision $x$ that minimizes $\Fg(\cdot,k)$ for all sufficiently large $k$ is a minimal element of $\wprefrb$, and we therefore seek the minimal elements of $\wprefrb$ by minimizing $\Fg$ with respect to $x$.

One option is to solve the static objective $\Fg$, %
\begin{equation}
    \xstar \in \argmin_{x \in \cX} \quad \Fg(x,k), \quad \text{for any fixed } k > k_0.
    \label{eq:static}
\end{equation} 
However, this presents three challenges. First, the threshold $k_0$ in~\eqref{eq:param_repr_condt} depends on the pair $(x,y)$, and in general no finite $k_0$ works uniformly over $\cX$, so for any finite $k$ the function $\Fg$ does not represent $\ordX$. Second, %
$\Fg$ becomes ill-conditioned at large $k$ due to large gradient variations, rendering solvers numerically unstable. Third, initializing with a large $k$ risks the same myopia found in preemptive approaches, where the large relative gradients of high-priority rules preclude exploration considering lower-priority rules. 

Instead, we propose to increase $k$ gradually during the minimization of $\Fg$. Since every rule is present in the objective throughout, the solver can exploit the structure of the lower-priority rules while descending the higher-priority ones. %
Concretely, \Cref{al:algorithm_1}, \emph{Lexicographic Descent}, interleaves steps of a base solver with updates of $k$, where a \emph{base solver} is any iterative method for minimizing a fixed objective, such as projected gradient descent. We write $\SolverStep(y, g(k))$ for one base-solver iteration on $\Fg(\cdot,k)$, and at outer iteration $k$ perform $T(k)$ such steps with fixed weights $g(k)$, where $T : \bbN \to \Npos$ is an inner-step schedule and $g$ the weight schedule.
\begin{algorithm}[t]
    \caption{Lexicographic Descent}
    \label{al:algorithm_1}
    \small
    \begin{algorithmic}[1]
        \NoNumberInput $\xinit$, $g$, \SolverStep, $T$, \Converged
        \State $x_0 \gets \xinit$, \quad $k \gets 0$
        \Repeat
            \State $k \gets k + 1$, \quad $y \gets x_{k-1}$
            \State \algorithmicfor\ $t = 1, \dots, T(k)$ \algorithmicdo\ $y \gets \SolverStep(y, g(k))$
            \State $x_k \gets y$
        \Until{$\Converged(x_k, x_{k-1}, k)$}
        \State \Return $x_k$
    \end{algorithmic}
\end{algorithm}

We now show that whenever each outer iteration achieves a sufficient decrease, the iterates converge, for any inner-step schedule $T$, to the set of constrained stationary points of the highest-priority rule, $r_0$; this matches the guarantee that the preemptive approach provides.

To ease the analysis, we work with the following normalized objective
\begin{equation}
    \Ft(x,k) \defeq \frac{\Fg(x,k)}{g_0(k)} = r_0(x) + \sum_{i=1}^{N-1} w_i(k)\, r_i(x), \label{eq:norm_obj}
\end{equation}
with $w_i(k) \defeq g_i(k)/g_0(k)$, which satisfies $\limk w_i(k) = 0$ for every $i \ge 1$ by~\eqref{eq:limit_cond}, and write $\dsumj{k} \defeq \sum_{i \ge j} w_i(k)$ for the total weight from level $j$ on, $\dsum{k} \defeq \sum_{i \ge 1} w_i(k)$. Since $g_0(k) > 0\, \forall k$, both objectives have the same stationary points, so the analysis applies equally to $\Fg(\cdot,k)$.
As $\cX$ is constrained, minimizers of $\Ft(\cdot,k)$ may lie on its boundary, where $\nabla \Ft(\xstar, k) \neq 0$; the relevant notion is therefore constrained stationarity.

\begin{definition}[Constrained stationary point~\cite{Nocedal2006}]\label{def:constr_stat_point}
Let $\tcone{x}$ denote the tangent cone to $\cX$ at $x \in \cX$. A point $x \in \cX$ is called a \emph{constrained stationary point} of a function $f$ differentiable at $x$ if
\[
    \nabla f(x)^\top d \ge 0,
    \qquad \forall d \in \tcone{x}.
\]
\end{definition}

We write $\Sof{f}$ for the set of constrained stationary points of a function $f$, and $\Szero \defeq \Sof{r_0}$, the superscript $0$ referring to the limiting objective $\limk \Ft(\cdot,k) = r_0$. For a nonempty closed set $C \subseteq \cH$, $\dist(x, C) \defeq \inf_{y \in C} \|x - y\|$. We make the following assumptions.

\begin{assumption}[Convex decision space]\label{ass:convex}
$\cX$ is nonempty, compact, and convex.
\end{assumption}

\begin{assumption}[$C^1$ rules]\label{ass:c1}
Each rule $r_i$ is $C^1$ on an open set containing $\cX$.
\end{assumption}

\begin{assumption}[Eventually monotone weight schedule]\label{ass:monotone_schedule}
There exists $K_1 \in \bbN$ such that $w_i(k) \geq w_i(k+1)$, $\forall i \ge 1$, $k\ge K_1$.
\end{assumption}

These assumptions are not restrictive in practice. Regarding~\Cref{ass:convex}, the decision variable is the control input, whose admissible set is typically convex~\cite{Borrelli2017}, e.g., box constraints on the inputs; state constraints and obstacles, which would render the decision space nonconvex, are instead expressed as rules. We believe the analysis extends to nonconvex decision spaces, and leave this for future work.
Secondly,~\Cref{ass:c1} may initially appear restrictive. However, rules of the form $\max\{0, e(x)\}^2$ with $e$ a $C^1$ violation measure cover various rules naturally: speed limits, lane keeping, and stopping at red lights are $C^1$; rules involving a minimum over several violation measures, such as collision avoidance with distance to various entities of interest, or rules with decision-dependent discrete logic, such as right of way at a stop sign, are at best piecewise $C^1$ and require a smooth approximation. We expect the analysis to extend to more generic rules; this extension is left for future work.
Finally, although \eqref{eq:limit_cond} does not imply~\Cref{ass:monotone_schedule}, the only schedules that are excluded are those whose weights vanish while oscillating.

For the base solver, we restrict attention to a subfamily of gradient-based solvers; extensions to other families of solvers are left for future work. 
To define the subfamily, we first introduce the projected-gradient mapping $\Gof{f}(x)$ for a $C^1$ function $f$,
\[
    \Gof{f}(x) \defeq \Lt\Big( x - \proj\big( x - \tfrac{1}{\Lt} \nabla f(x) \big) \Big),
\]
with $\proj$ the Euclidean projection onto $\cX$ and $\Lt > 0$ a fixed constant; abbreviate $\Gk \defeq \Gof{\Ft(\cdot,k)}$ and $\Gzero \defeq \Gof{r_0}$.

We now introduce the requirements for the base solver.

\noindent Write $\InnerLoop(x,k)$ for the composition of the $T(k)$ base-solver steps performed at outer iteration $k$ (Line~4 of~\Cref{al:algorithm_1}), and $x_k$ for the iterate of~\Cref{al:algorithm_1} produced by outer iteration $k$, so that $x_k = \InnerLoop(x_{k-1}, k)$ for $k \ge 1$, with $x_0 = \xinit$.

\begin{definition}[Descent loop]\label{def:descent_step}
A map $\InnerLoop : \cX \times \bbN \to \cX$ is a \emph{descent loop} if there is a continuous function $\rho : \Rnn \to \Rnn$ with $\rho(0) = 0$ and $\rho(t) > 0$ for $t > 0$ such that $\InnerLoop(x,k) \in \cX$ for all $x \in \cX$, $k \in \bbN$, and there is $K_2 \in \bbN$ such that, for all $x \in \cX$ and $k > K_2$,
\begin{equation}
    \Ft\big(\InnerLoop(x,k), k\big) \;\le\; \Ft(x,k) - \rho\big(\|\Gk(x)\|\big) .
    \label{eq:sufficient_decrease}
\end{equation}
\end{definition}

\noindent \Cref{def:descent_step} plays the role of the sufficient-decrease lemma for projected gradient~\cite[Lem.~10.4]{doi:10.1137/1.9781611974997}, but is more general in two respects: it constrains the composition of the $T(k)$ base-solver steps rather than each step, and it quantifies the decrease by an arbitrary forcing function $\rho$ rather than by a fixed multiple of $\|\Gk(\cdot)\|^2$.
Projected gradient with Armijo backtracking from the trial step $1/\Lt$ is a descent loop. Although \Cref{ass:c1} requires $C^1$ rules, it does not preclude higher-order methods. If the rules are additionally $C^2$, then every $\Ft(\cdot,k)$ has a Lipschitz gradient with a constant $L$ uniform in $k$, and for any $\Lt \ge L$, projected gradient with the fixed step $1/\Lt$ is a descent loop with $\rho(t) = t^2/(2\Lt)$, and second-order methods, such as projected regularized Newton, become available.

To prove the results in~\Cref{thm:set_convergence},~\Cref{lem:pg,lem:invert} are required. The former collects standard properties of the projected-gradient mapping, and the latter states that, on a compact set, a continuous function can only become arbitrarily small near the points where it vanishes; it is a standard compactness argument, cf.~\cite[Lem.~4.3]{Khalil2002}.
\begin{lemma}[Projected-gradient facts~{\cite[Thms.~6.42 and~10.7]{doi:10.1137/1.9781611974997}}]\label{lem:pg}
Let \Cref{ass:convex} hold, and let $f, h$ be $C^1$ on an open set containing $\cX$.
\begin{enumerate}[label=(\roman*), leftmargin=*, nosep]
    \item $x \in \Sof{f}$ if and only if $\Gof{f}(x) = 0$; 
    \item $\|\Gof{f}(x) - \Gof{h}(x)\| \le \|\nabla f(x) - \nabla h(x)\|$,  $\forall x \in \cX$;
    \item $\Gof{f}$ is continuous, and $\Sof{f}$ is compact.
\end{enumerate}
\end{lemma}

\begin{lemma}\label{lem:invert}
Let $M$ be a compact metric space, let $f : M \to \Rnn$ be continuous, and let $Z \defeq \{x \in M : f(x) = 0\}$ be nonempty. Given a sequence $\{y_k\}_k \subseteq M$ that satisfies $\limk f(y_k) = 0$, then $\limk \dist(y_k, Z) = 0$.
\end{lemma}

\begin{theorem}[Set convergence]\label{thm:set_convergence}
    Let \Cref{ass:convex,ass:c1,ass:monotone_schedule} hold, $\xinit \in \cX$ be any initialization, and $\InnerLoop$ be any descent loop (\Cref{def:descent_step}) for any inner-step schedule $T$. Then the sequence $\{x_k\}_{k \in \bbN}$ generated by~\Cref{al:algorithm_1} satisfies $\limk \dist(x_k, \Szero) = 0$.
\end{theorem}

\begin{proof}
    The proof is conducted in two parts. First we prove that $\sum_k \rho(\|\Gk(x_{k-1})\|) < \infty$, and hence that $\limk \|\Gk(x_{k-1})\| = 0$. Second, we prove that $\Gk$ converges uniformly to $\Gzero$ on $\cX$, and since $\limk \|\Gzero(x_{k-1})\| = 0$, $\limk \dist(x_k, \Szero) = 0$.
For the remainder of the proof, let $K_2$ and $\rho$ be as in~\Cref{def:descent_step} for the given $\InnerLoop$, $K_1$ as in~\Cref{ass:monotone_schedule}, $\Kbar \defeq \max\{K_1, K_2\}$, $\Gbar \defeq \max_{i, x \in \cX} \|\nabla r_i(x)\|$ and $\rho_k \defeq \rho(\|\Gk(x_{k-1})\|)$.

\emph{Step 1 ($\limk \|\Gk(x_{k-1})\| = 0$).} For every $x \in \cX$ and every $k \ge \Kbar$,
\[
    \Ft(x, k{+}1) - \Ft(x, k) = \sum_{i \ge 1} \big( w_i(k{+}1) - w_i(k) \big)\, r_i(x) \le 0 .
\]
Applying~\Cref{def:descent_step} at $x_{k-1}$, and with $x = x_{k-1}$, gives, for every $k > \Kbar$,
\[
    \Ft(x_k, k) \;\le\; \Ft(x_{k-1}, k) - \rho_k \;\le\; \Ft(x_{k-1}, k{-}1) - \rho_k ,
\]
that is, $\rho_k \le \Ft(x_{k-1}, k{-}1) - \Ft(x_k, k)$. Fix $K > \Kbar$, summing over $\Kbar < k \le K$,
\begin{align*}
    \sum_{k = \Kbar+1}^{K} \rho_k
    &\le \sum_{k = \Kbar+1}^{K} \big( \Ft(x_{k-1}, k{-}1) - \Ft(x_k, k) \big) \\
    &= \Ft(x_{\Kbar}, \Kbar) - \Ft(x_{K}, K) \\
    &\le \Ft(x_{\Kbar}, \Kbar) ,
\end{align*}
the last inequality holds since $\Ft(x_{K}, K) \ge 0$. Letting $K \to \infty$,
$
    \sum_{k > \Kbar} \rho_k \;\le\; \Ft(x_{\Kbar}, \Kbar) \;<\; \infty ,
$
and since $\rho_k \ge 0$ for every $k$, $\limk \rho_k = 0$. Since $w_i(k) \le w_i(\Kbar)$, by \Cref{lem:pg}(ii), with $h(x) = 0, \forall x \in \cX$, for every $x \in \cX$ and every $k \ge \Kbar$,
\[
    \|\Gk(x)\| \;\le\; \|\nabla \Ft(x,k)\| \;\le\; \Gbar\,(1 + \dsum{k}) \;\le\; \Gbar\,(1 + \dsum{\Kbar}) .
\]
In particular $\|\Gk(x_{k-1})\| \in [0,\, \Gbar(1 + \dsum{\Kbar})]$ for every $k > \Kbar$, and $\rho(\|\Gk(x_{k-1})\|) = \rho_k \to 0$ as $k \to \infty$; applying~\Cref{lem:invert} with $M = [0,\, \Gbar(1 + \dsum{\Kbar})]$, $f = \rho$, and $Z = \{0\}$ (nonempty since $\rho(0)=0$) to the sequence $y_k \defeq \|\Gk(x_{k-1})\|$ gives $\limk \|\Gk(x_{k-1})\| = 0$.

\emph{Step 2 ($\limk \dist(x_k, \Szero) = 0$).} By \Cref{lem:pg}(ii),
\begin{align*}
    \sup_{x \in \cX} \| \Gk(x) - \Gzero(x) \|
    &\le \sup_{x \in \cX} \big\| \nabla \Ft(x,k) - \nabla r_0(x) \big\| \\
    &= \sup_{x \in \cX} \Big\| \sum_{i \ge 1} w_i(k) \nabla r_i(x) \Big\| \\
    &\le \Gbar \dsum{k} .
\end{align*}
Since $\limk \dsum{k} = 0$, $\Gk$ converges uniformly to $\Gzero$ on $\cX$. By the triangle inequality,
\begin{align*}
    \big\| \Gzero(x_{k-1}) \big\|
    &\le \big\| \Gk(x_{k-1}) - \Gzero(x_{k-1}) \big\| + \big\| \Gk(x_{k-1}) \big\| \\
    &\le \Gbar\, \dsum{k} + \big\| \Gk(x_{k-1}) \big\| .
\end{align*}
The right-hand side tends to $0$ by Step~1, hence $\limk \|\Gzero(x_{k-1})\| = 0$. Since $\cX$ is compact by~\Cref{ass:convex}; $\|\Gzero(\cdot)\|$ is continuous and nonnegative by~\Cref{lem:pg}(iii); its zero set is $\{x \in \cX : \Gzero(x) = 0\} = \Szero$ by~\Cref{lem:pg}(i); and $\Szero \neq \emptyset$, since $r_0$ is continuous on the compact set $\cX$ and therefore attains a minimizer over $\cX$.
We similarly apply~\Cref{lem:invert} with $M = \cX$, $f = \|\Gzero(\cdot)\|$ and $Z = \Szero$.
Since $\{x_k\}_k \subseteq \cX$ and $\limk \|\Gzero(x_{k-1})\| = 0$, \Cref{lem:invert} yields $\limk \dist(x_{k-1}, \Szero) = 0$, and hence $\limk \dist(x_k, \Szero) = 0$.
\end{proof}

\Cref{thm:set_convergence} holds for every choice of the inner-step schedule $T$. This independence holds because as $k \to \infty$ the weights $w_i(k)$ vanish, so that $\Ft(\cdot,k) \to r_0$ and the lower-priority rules no longer influence the objective.

However, $\Szero$ generally contains many points, which the lower-priority rules distinguish. \Cref{thm:set_convergence} does not state which of them~\Cref{al:algorithm_1} approaches. 
A more desirable target is the set of limit points of the stationary points of the weighted objectives,
\[
    \Slim \defeq \big\{ x \in \cX : \liminfk \dist(x, \Sk) = 0 \big\},
\]
which retains only those stationary points of $r_0$ that are also balanced by the lower-priority rules; in particular, any decision satisfying every rule belongs to $\Sk$ for all $k$ and hence to $\Slim$. Indeed, since $\Gk \to \Gzero$ uniformly, $\Slim \subseteq \Szero$. However, whether~\Cref{al:algorithm_1} converges to $\Slim$ depends on the choice of $T$ and $g$, which is the subject of the next section.

\section{A Necessary Condition for Escaping Indifferent Regions} \label{sec:travel_budget}

Convergence of~\Cref{al:algorithm_1} to $\Slim$ is independent of $T$ and $g$ in two cases: when $\Szero = \Slim$, so that~\Cref{thm:set_convergence} already gives $\limk \dist(x_k, \Slim) = 0$, and when the descent loop returns a minimizer of $\Ft(\cdot,k)$ at every $k$, so that $x_k \in \Sk$ for all $k$. Outside these cases the selection of the schedulers matters. We now derive a necessary condition on $T$ and $g$ for~\Cref{al:algorithm_1} to converge to $\Slim$.

To understand the obstruction quantified by the necessary condition, consider a region of $\cX$ where the higher-priority rules are stationary. The gradient of $\Ft(\cdot,k)$ is supplied by the lower-priority rules. However, as $k$ increases, the relative contribution of the lower-priority rules vanishes, and thus $\nabla \Ft(\cdot, k)$ tends to zero. If the weights go to zero too quickly, then the iterates are not able to leave this region, and thus cannot reach $\Slim$. We now formalize this reasoning and provide a necessary condition to converge to $\Slim$. Throughout, let $V \subseteq \cH$ be a subspace with orthogonal projection $\projon{V}$.

\begin{definition}[Indifferent region]\label{def:indifferent}
Let $1 \le j \le N-1$. A set $U \subseteq \cX$ is an \emph{indifferent region at level $j$ along $V$} if $\projon{V} \nabla r_i(x) = 0$ for every $x \in U$ and every $i < j$.
\end{definition}

The obstruction arises for descent loops whose displacement along $V$ is limited; we formalize this as follows.
\begin{assumption}\label{ass:rate_limited}
Let $\gamma_k(x,\eta) \defeq \max\{ \|\projon{V} \nabla \Ft(y,k)\| : y \in \cX,\ \|\projon{V}(y - x)\| \le \eta \}$.
There is $\tau > 0$ such that $\forall x \in \cX, k \in \bbN, \eta > 0$,
\begin{align*}
    &\tau\, T(k)\, \gamma_k(x,\eta) \le \eta \\
    &\qquad\Longrightarrow\;\;
    \|\projon{V}(\InnerLoop(x,k) - x)\| \le \tau\, T(k)\, \gamma_k(x,\eta) .
\end{align*}
\end{assumption}

\noindent \Cref{ass:rate_limited} says that the inner loop moves along $V$ no further than $T(k)$ steps taken in a straight line, each of length at most $\tau$ times the largest $V$-gradient nearby. The assumption is not too restrictive: projected gradient with step size $\tau = 1/\Lt$ and Armijo backtracking with initial trial step $\tau$ both satisfy it, for instance, when $\cX$ is a box and $V$ a coordinate subspace.

We now state the necessary condition for convergence to $\Slim$.

\begin{proposition}[Necessary condition for convergence to $\Slim$]\label{prop:travel_budget}
Let \Cref{ass:convex,ass:c1,ass:rate_limited} hold, $U$ be an indifferent region at level $j$ along $V$, and $d_U \defeq \inf\{\|\projon{V}(y - \xinit)\| \mid y \in \cX \setminus U\} \in [0,\infty]$. If the iterates of~\Cref{al:algorithm_1} started at $\xinit$ satisfy $\lim_{k\to\infty}\dist(x_k, \Slim) = 0$, then
\begin{equation}\label{eq:travel_budget}
    \tau\, \Gbar_j \sum_{k=1}^{\infty} T(k)\, \dsumj{k} \;\ge\; \min\big\{ d_U,\; \dist(\projon{V}\xinit, \projon{V}\Slim) \big\} ,
\end{equation}
where $\Gbar_j \defeq \max_{i \ge j, x \in \cX} \|\projon{V}\nabla r_i(x)\|$, assumed positive. %
\end{proposition}

\begin{proof}
Let $B_k \defeq \tau \Gbar_j \sum_{m=1}^{k} T(m)\, \dsumj{m}$, $B_\infty \defeq \lim_k B_k$, and ${\eta_k \defeq B_{k+1} - B_k} = {\tau T(k{+}1)\, \Gbar_j\, \dsumj{k+1}} > 0$. The proof is by contrapositive: if $B_\infty < \min\{d_U, \dist(\projon{V}\xinit, \projon{V}\Slim)\}$, then $\dist(x_k, \Slim) \not\to 0$ as $k \to \infty$. We first show that $\|\projon{V}(x_k - \xinit)\| \le B_k \le B_\infty$, $\forall k$ by induction; the case $k = 0$ holds since $x_0 = \xinit$.
For the induction step, suppose $\|\projon{V}(x_k - \xinit)\| \le B_k$ for some $k$.
Every $y \in \cX$ with $\|\projon{V}(y - x_k)\| \le \eta_k$ satisfies $\|\projon{V}(y - \xinit)\| \le \eta_k + B_k = B_{k+1} \le B_\infty < d_U$ (by assumption). Hence $y \in U$, where $\projon{V}\nabla r_i(y) = 0$, $\forall i < j$, and $\|\projon{V}\nabla\Ft(y,m)\| \le \Gbar_j\, \dsumj{m}$, $\forall m \in \bbN$, in particular for $m = k{+}1$. Thus, $\gamma_{k+1}(x_k, \eta_k) \le \Gbar_j\, \dsumj{k+1}$ by definition of $\gamma_{k+1}$, so $\tau T(k{+}1)\, \gamma_{k+1}(x_k, \eta_k) \le \tau T(k{+}1)\, \Gbar_j\, \dsumj{k+1} = \eta_k$. \Cref{ass:rate_limited} gives $\|\projon{V}(x_{k+1} - x_k)\| \le \tau T(k{+}1)\, \gamma_{k+1}(x_k, \eta_k) \le \eta_k$, whence $\|\projon{V}(x_{k+1} - \xinit)\| \le B_k + \eta_k = B_{k+1}$, which completes the induction. Hence, for every $k$ and every $s \in \Slim$, $\|x_k - s\| \ge \|\projon{V}(x_k - s)\| \ge \dist(\projon{V}\xinit, \projon{V}\Slim) - B_\infty$; the right-hand side is a positive constant independent of $k$, so $\dist(x_k, \Slim) \not\to 0$, which proves the contrapositive.
\end{proof}
\Cref{prop:travel_budget} states that~\eqref{eq:travel_budget} is a necessary condition to converge to $\Slim$. $T$ and $g$ enter only the left-hand side of~\eqref{eq:travel_budget}, in particular $g$ through the weights $w_i(k)$ and hence $\dsumj{k}$. The right-hand side, in contrast, depends on the problem alone, through $U$, $\xinit$ and $\Slim$, which are generally unknown a priori.
One way to ensure satisfaction of~\eqref{eq:travel_budget} for every $j$, $U$ and $\xinit$ is to make its left-hand side diverge by setting $\sum_{k=1}^\infty T(k)\, w_{N-1}(k) = \infty$, since $\dsumj{k} \ge w_{N-1}(k)$ for every $j$.
The choice thus covers every run, whether the iterates start in $U$, enter $U$ later (applying~\Cref{prop:travel_budget} from the iteration at which they enter), or never encounter $U$.

\section{Numerical Case Studies} \label{sec:MPP}

We present two case studies: a two-dimensional toy problem, which verifies~\Cref{thm:set_convergence} and illustrates the design of the schedulers with~\Cref{prop:travel_budget}, and motion planning for an~\gls{av}, on which we compare Lexicographic Descent against the preemptive approach.

In both case studies we minimize~\eqref{eq:norm_obj}. The base solver is projected gradient descent, the inner-step schedule is $T(k) = 1$, $\forall k$, to isolate the effect of the weight scheduler, and $g$ is given by $g_i(k) \defeq \la(k)^{N-i-1}$, with $\la: \bbN \to \Rpos$ specified in the respective case studies.

\subsection{Two-Dimensional Toy Problem} \label{sec:toy_examples}

\subsubsection{Setup}
The decision space is $\cX = [-2, 2]^2$ and the rulebook comprises the two rules
$
    r_0(x) = (x_2^2 - 1)^2, r_1(x) = (x_1 - 1)^2 + (x_2 - 2)^2 ,
$
shown in~\Cref{fig:toy_setup} (left). Both~\Cref{ass:convex,ass:c1} hold. The stationary set of $r_0$ is $\Szero = \{x \in \cX : x_2 \in \{-1, 0, 1\}\}$. The lexicographic minimum is $(1,1)$, and $\Slim = \{(1,-1), (1,0), (1,1)\}$.

\subsubsection{Problem properties}
With $N = 2$, the objective~\eqref{eq:norm_obj} has weights $w_0(k) = 1, \forall k$, $w_1(k) = 1/\la(k)$; and $\Sk = \{x \in \cX : x_1 = 1,\; 4x_2(x_2^2 - 1) + 2 w_1(k)(x_2 - 2) = 0\}$. For $w_1(k) > \wfold \approx 0.3013$, $\Sk$ is a singleton; below $\wfold$ two additional stationary points appear, see~\Cref{fig:toy_setup} (right).

\subsubsection{Solver settings}
The solver uses step size $1/\Lt$, where $\Lt$ is the Lipschitz constant of $\nabla \Ft(\cdot, k)$ uniform in $k$. We compare three weight schedules,
$
    \la_1(k) = 1.1^k,  \la_2(k) = 1 + 0.1\, k, \la_3(k) = (k + 1)^{1/4}.
$
All three satisfy~\eqref{eq:limit_cond} and~\Cref{ass:monotone_schedule}.
\begin{figure}
    \centering
    \includegraphics[width=\linewidth]{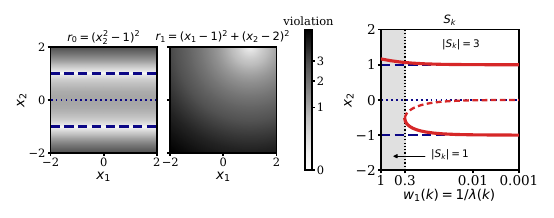}
    \caption{The toy problem. Left: the rules $r_0$ and $r_1$, shaded by violation score. Overlaid on $r_0$ is $\Szero$, valleys dashed and ridge dotted. Right: the $x_2$-coordinates of $\Sk$ against $w_1(k)$ ($x_1 = 1$ throughout). For $w_1(k) > \wfold$ (shaded) $\Sk$ is a singleton in the upper valley; below, the lower minimum (solid) and the saddle (dashed) arise.}
    \label{fig:toy_setup}
\end{figure}
\begin{figure}
    \centering
    \includegraphics[width=0.75\linewidth]{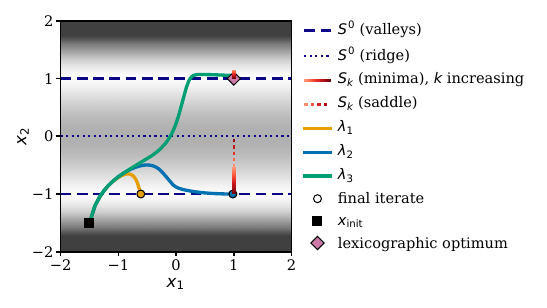}
    \caption{Executed paths from $\xinit = (-1.5, -1.5)$ for $10^6$ solver steps. Red lines are $\Sk$, solid for the minima and dashed for the saddle, darkening as $k$ grows. Shading matches $r_0$ for display of the valleys.}
    \label{fig:toy_paths}
\end{figure}
\subsubsection{Results}
\Cref{fig:toy_paths} shows the executed paths from $\xinit = (-1.5,-1.5)$ over $10^6$ solver steps, and \Cref{fig:toy_tracking} the distances to the valleys $\{x_2 = \pm 1\} \subset \Szero$ and to $\Sk$; \Cref{fig:toy_robustness} repeats the runs from $10^4$ initializations drawn uniformly on $\cX$. All runs confirm~\Cref{thm:set_convergence}, and differ in which $x \in \Szero$ they reach. 
Under $\la_1$, the iterates do not reach $\Slim$ from $96\%$ of the initializations and reach the lexicographic optimum from only $2.5\%$. \Cref{prop:travel_budget} predicts that $42\%$ of the initializations cannot converge to $\Slim$, and indeed none of them did. The remaining initializations satisfy~\eqref{eq:travel_budget} but did not converge to $\Slim$ either, highlighting that the condition is necessary but not sufficient.
Under $\la_2$ and $\la_3$, which satisfy~\eqref{eq:travel_budget}, iterates reach $\Slim$ always, and the lexicographic optimum on $73.5\%$ and $100\%$ of initializations respectively. The difference is which branch of $\Sk$ is followed once $\Sk$ bifurcates at $w_1(k) < \wfold$ ($k = 24$ for $\la_2$, $k = 121$ for $\la_3$). $\la_3$ decays slowly enough for the iterates to cross the ridge at $x_2 = 0$ while $\Sk$ is still a singleton, whereas under $\la_2$ those started with $(\xinit)_2 < -0.9$, such as $\xinit$, cannot and lock onto $(1,-1) \in \Slim$. 

Although convergence was always to the lexicographic optimum for $\la_3$, this property need not hold in general, as the first branch to appear may lead to a local minimum. Nevertheless, sufficient conditions on $T$ and $g$ that exploit such structure are an interesting avenue for future work.

\begin{figure}
    \centering
    \includegraphics[width=0.9\linewidth]{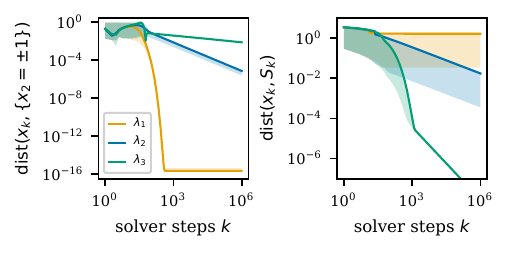}
    \caption{Distance of the iterates to $\{(x_1, x_2)\mid x_2 = \pm1 \}$ (left) and to $\Sk$ (right). Lines are $\xinit = (-1.5,-1.5)$ of~\Cref{fig:toy_paths}; shading indicates $95\%$ quantile bands over $10^4$ initializations drawn uniformly on $\cX$.}
    \label{fig:toy_tracking}
\end{figure}
\begin{figure}
    \centering
    \includegraphics[width=\linewidth]{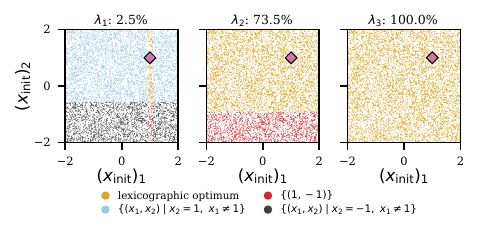}
    \caption{Robustness to the initialization. Each point is an initialization $x \in \cX$, colored by where the final iterate lies: at the lexicographic optimum, elsewhere in the upper valley, at the local minimum $(1,-1)$, or elsewhere in the lower valley; percentages give the fraction reaching within $\epsilon=5\cdot10^{-2}$ of the lexicographic optimum. %
    }
    \label{fig:toy_robustness}
\end{figure}
\subsection{Autonomous Driving} \label{sec:ad_examples}
We now apply~\Cref{al:algorithm_1} to motion planning for an~\gls{av} and compare it with the preemptive approach~\cite{Sherali1983}. We show that \Cref{al:algorithm_1} attains on average lexicographically no larger rule violations than the preemptive approach in nominal driving, and is lexicographically preferable in $13$ of $15$ variants of a cut-in scenario, substantiating that~\Cref{al:algorithm_1} is less susceptible to suboptimal decisions. %

\subsubsection{Setup}
The ego vehicle drives on scenarios from the Waymo Open Motion Dataset~\cite{Ettinger2021} in closed loop in the waymax simulator~\cite{Gulino2023}; the other agents replay their logs. 
Every $0.1$\,s the planner solves the lexicographic problem over a horizon of $H=4$\,s, the decision $x \in \cX$ being the acceleration and curvature inputs over the horizon. 
The rulebook comprises five $C^1$ rules adapted from~\cite{Xiao2021rules}: collision avoidance (with clearance), staying on the correct carriageway, speed limit, arrival at the destination at a specified time, and lane maintenance; the collision and lane rules are smoothed. 
The base solver is projected gradient descent with Armijo backtracking, which satisfies~\eqref{eq:sufficient_decrease}, and $\la(k) = (k+1)^{1/8}$. 
Each planning step runs $K = 480$ base-solver steps for either method, so that violations are compared at the same number of steps. 
In practice, a termination criterion such as a small residual $\|\Gk(x_k)\|$ could be used, but it would have to be tied to a preselected $K$, since $\Ft(\cdot,k)$ changes with $k$. The preemptive approach sequentially minimizes $r_i$ subject to $r_j(x) \le r_j(\xstar_j)$, $j < i$, with the same base solver. 

\subsubsection{Nominal Driving}
We run both methods on 150 scenarios whose logged trajectory satisfies $r_0$--$r_2$. Both methods are warm-started from the previous plan, as is standard in receding-horizon planning, the first plan being initialized at zero.
The remaining rules need not be satisfiable: $r_4$ is violated by any lane change and $r_3$ depends on the traffic, so the set mixes scenarios in which every rule can be satisfied with scenarios in which some cannot. \Cref{tab:av_results} (top) reports the realized violations of each rule, averaged over the steps of the trajectory and over scenarios. A mean of $1$ corresponds to a violation of $1$ at every step of every scenario; for the terminal rule $r_3$ it is the final-step violation divided by the number of steps.
Both methods satisfy all five rules in most scenarios: at a tolerance of $10^{-3}$ per rule they tie in $117$ of $150$. The preemptive violations are caused by a few scenarios; in two of them it commits, at its lane level, to a stationary point with nonzero violation, leaves the road and stops, which the lower levels cannot undo. The following scenario highlights this susceptibility.

\begin{table}
    \centering
    \footnotesize
    \caption{Realized rule violations averaged over the steps of each trajectory and over scenarios.}
    \label{tab:av_results}
    \setlength{\tabcolsep}{3pt}
    \begin{tabular}{lccccc}
        \toprule
        & $r_0$ & $r_1$ & $r_2$ & $r_3$ & $r_4$ \\
        \midrule
        \multicolumn{6}{l}{\emph{Nominal driving, $150$ scenarios}} \\
        Ours & $5.9{\cdot}10^{-4}$ & $3{\cdot}10^{-5}$ & $1.1{\cdot}10^{-3}$ & $1.5{\cdot}10^{-3}$ & $6.8{\cdot}10^{-3}$ \\
        Preemptive & $2.8{\cdot}10^{-3}$ & $0.19$ & $8.4{\cdot}10^{-3}$ & $0.19$ & $0.33$ \\
        \midrule
        \multicolumn{6}{l}{\emph{Cut-in scenario, $15$ variants}} \\
        Ours & {$3.8{\cdot}10^{-3}$} & {$0.051$} & {$3.5{\cdot}10^{-4}$} & {$0.26$} & {$0.25$} \\
        Preemptive & {$5.6{\cdot}10^{-5}$} & {$0$} & {$0$} & {$1.2$} & {$0.026$} \\
        \bottomrule
    \end{tabular}
\end{table}

\subsubsection{Cut-in Scenario}
We construct a parametrized cut-in scenario from the same dataset (\Cref{fig:emergency_cutin}): a traffic jam occupies the ego vehicle's lane and the lane to its right, and a vehicle from the left lane cuts in ahead of the ego and brakes to join the traffic jam. We vary the ego vehicle's initial position along its lane by $\{-2,\dots,2\}$\,m and its initial speed by $\{-1,0,1\}$\,m/s, giving $15$ variants; the other vehicles are unchanged. The destination is occupied by a stationary vehicle, so $r_3$ cannot be satisfied. The ego vehicle can satisfy $r_0$--$r_2$ either by stopping behind the traffic, which increases the arrival violation $r_3$, or by moving into the left lane, which reduces it at the cost of violating the lane-maintenance rule $r_4$; since $r_3$ ranks above $r_4$, the lane change is lexicographically preferred. Unlike in nominal driving, both methods are initialized at zero at every planning step, so that the previous plan cannot bias the current decision. \Cref{al:algorithm_1} changes lane in $14$ variants and, compared lexicographically with a tolerance of $10^{-3}$ per rule (\Cref{tab:av_results}, bottom), is preferable in $13$; in one it overshoots the speed limit by $1.15$\,m/s for $0.2$\,s, in the other the cutting-in vehicle pushes it onto the shoulder, violating $r_1$. The preemptive approach stays behind in all $15$, which suggests that it is susceptible to local solutions and motivates keeping all rules in the objective throughout; the two exceptions, in turn, show the importance of tuning the solver to avoid spurious solutions.

\begin{figure}[t]
    \centering
    \includegraphics[width=0.9\linewidth]{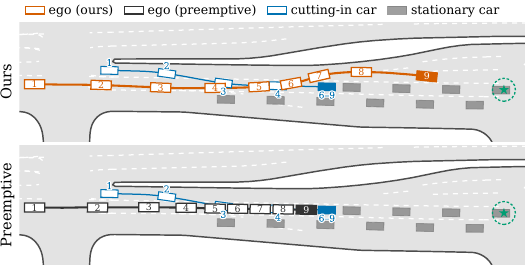}
    \caption{The cut-in scenario. A vehicle leaves the free lane and brakes to join the traffic jam in the ego vehicle's lane; boxes show the vehicles every second, labeled by time in seconds, and the star marks the destination. Ours (top) changes into the free lane and passes; the preemptive approach (bottom) stays behind the vehicle.}
    \label{fig:emergency_cutin}
\end{figure}

\section{Conclusions and Future Work}\label{sec:conclusion}

We introduced a weighted sum that parametrically represents a lexicographic order, and Lexicographic Descent, which interleaves base-solver steps with weight updates. The iterates converge to the constrained stationary set of the highest-priority rule, and we derived a necessary condition on the schedulers for Lexicographic Descent to converge to the limit set of the set of stationary points of the parametric objective. Because all rules remain in the objective, lower-priority rules guide exploration: in the autonomous driving case study, Lexicographic Descent outperformed the preemptive approach, which stalled at suboptimal solutions.

Several directions remain open. First, in the toy problem a suitable decay rate reached the lexicographic optimum from every sampled initialization; sufficient conditions on the schedules that depend on problem structure may turn this into a guarantee. Second, rules that are not naturally $C^1$ can be smoothed, as in our case study; the unsmoothed case is still to be analyzed. Third, nonconvex decision spaces, arising from state constraints or obstacles imposed on the decision, are not covered by our assumptions.

\bibliographystyle{IEEEtran}
\bibliography{references_mendeley}

\end{document}